\documentclass{article}

\usepackage{microtype}
\usepackage{graphicx}
\usepackage{subfigure}
\usepackage{booktabs} 
\usepackage[hidelinks]{hyperref}

\usepackage{amsmath,amssymb,amsthm}
\usepackage{dsfont}
\usepackage[framemethod=TikZ]{mdframed}

\PassOptionsToPackage{usenames,dvipsnames}{xcolor}
\usetikzlibrary{arrows,positioning,automata,shapes,calc,patterns}

\newtheorem{assumption}{Assumption}
\newtheorem{theorem}{Theorem}
\newtheorem{lemma}{Lemma}

\DeclareMathOperator*{\supp}{supp}

\begin{document}

\title{Regularization Helps with Mitigating Poisoning Attacks: Distributionally-Robust Machine Learning Using the Wasserstein Distance }

\author{Farhad Farokhi\thanks{F. Farokhi is with the CSIRO's Data61 and the University of Melbourne, Australia. emails: farhad.farokhi@data61.csiro.au; ffarokhi@unimelb.edu.au}}

\maketitle

\begin{abstract} We use distributionally-robust optimization for machine learning to mitigate the effect of data poisoning attacks. We provide performance guarantees for the trained model on the original data (not including the poison records) by training the model for the worst-case distribution on a neighbourhood around the empirical distribution (extracted from the training dataset corrupted by a poisoning attack) defined using the Wasserstein distance. We relax the distributionally-robust machine learning problem by finding an upper bound for the worst-case fitness based on the empirical sampled-averaged fitness and the Lipschitz-constant of the fitness function (on the data for given model parameters) as regularizer. For regression models, we prove that this regularizer is equal to the dual norm of the model parameters. We use the Wine Quality dataset, the Boston Housing Market dataset, and the Adult dataset for demonstrating the results of this paper.
\end{abstract}

\section{Introduction}
Data poisoning attacks refer to adversarial attacks on machine learning by adding malicious entries to the training dataset in order to manipulate  the  model~\cite{vorobeychik2018adversarial}. The malicious data could be in the form of label modification or flipping (i.e., changing the outputs in supervised machine learning), data insertion attacks (i.e., adding a limited number of arbitrary data points), and data modification attacks (i.e., modifying some features or labels for an arbitrary subset of the data). These attacks have proved to be powerful in practice~\cite{biggio2011support,biggio2012poisoning}.

Regularization has been shown to be effective in training machine learning models that are robust against data poisoning attacks~\cite{demontis2017infinity}. This is motivated by that regularization can reduce the impact of contaminated data records during training. Regularization can make the models generalize better to data that is not directly in the training dataset. This is however only made for support vector machines. A general analysis of why regularization can be an effective method for mitigating the effect of data poisoning attacks is missing from the literature. This is investigated in this paper with the aid of distributionally-robust machine learning. 

In this paper, we show that, by using distributionally-robust optimization~\cite{esfahani2018data} for machine learning, we can combat  data poisoning attacks. We show that we can guarantee the  performance of a trained model on the original data (not including the poison records) by training the model for the worst-case distribution on a ball around the empirical distribution (extracted using the adversarially-manipulated training data) defined using the Wasserstein distance. The Wasserstein distance can be seen as the optimal mass transportation plan between two distributions~\cite{kantorovich1958space}. We find an upper bound for the worst-case expected fitness based on the empirical sampled-averaged fitness and the Lipschitz-constant of the fitness function on the data for given model parameters. This allows us to relax the distributionally-robust machine learning problem to optimizing  the empirical sampled-averaged fitness plus the Lipschitz constant as a regularizer. For regression models, we prove that this regularizer is equal to the dual norm of the model parameters. We use three different datasets, i.e., the Wine Quality dataset~\cite{cortez2009modeling}, the Boston Housing Market dataset~\cite{harrison1978hedonic}, and the Adult dataset~\cite{kohavi1996scaling} for demonstrating the results of this paper.

The rest of the paper is organized as follows. We discuss the related work on poisoning attacks in adversarial machine learning in Section~\ref{sec:related}. Background information on the Wasserstein distance is presented in Section~\ref{sec:Wasserstein}. In Section~\ref{sec:robustML}, we present the  distributionally-robust machine learning problem using the Wasserstein distance  and transform the distributionally-robust optimization problem into a regularized machine learning problem. We discuss the results in the context of linear and logistic regression in Section~\ref{sec:regression}. Finally, we present some numerical results in Section~\ref{sec:example} and conclude the paper in Section~\ref{sec:conc}.

\section{Related Work}\label{sec:related}
Machine learning under malicious noise has been investigated for a long time~\cite{valiant1985learning,kearns1993learning,bshouty2002pac,kalai2008agnostically,servedio2003smooth,steinhardt2017certified}. The idea of sub-sampling was proposed in~\cite{kearns1993learning}. This approach relies on taking several random sub-samples of the dataset and training machine learning models on them. Then, we use the model with the smallest training error. Sub-sampling is powerful when a sufficiently small fraction of the data is malicious~\cite{vorobeychik2018adversarial}. Approaches based on outlier removal are proposed in~\cite{klivans2009learning,cretu2008casting,barreno2010security}. They rely on identifying and removing a small fraction of the data that is anomalous (i.e., follows a different distribution than the rest).  Another approach is to use robust learning by trimmed optimization, which relies on minimizing the empirical risk while pruning a fraction of the data with the largest error~\cite{liu2017robust}. Finally, regularization has been used as a method for mitigating data poisoning attacks when training support vector machines~\cite{demontis2017infinity}. Although different in nature, regularization has been also shown to be useful for mitigating decision-time (evasion) attacks~\cite{li2014feature,li2018evasion}. 

Despite the success of regularization in adversarial machine learning, a general analysis of why it is an effective apparatus for mitigating data poisoning attacks is still missing from the literature. In addition, although Wasserstein distance has been used in machine learning~\cite{frogner2015learning,arjovsky2017wasserstein,blanchet2019robust,esfahani2018data}, it has not been used for designing distributionally-robust machine learning models against adversarial data poisoning attacks. 

An important aspect of our work is that we do not assume anything about the adversary, e.g., its cost function or motivation, except that there exists a bound on its ability in poisoning the training dataset, as opposed to studies relying on game-theoretic methods for adversarial machine learning~\cite{farokhi_svm_2019,zhang2017game,ou2019mixed} that need to model an adversary perfectly.

\section{Wasserstein Distance}\label{sec:Wasserstein}
Let $\mathcal{M}(\Xi)$ denote the set of all probability distributions $\mathbb{Q}$ on $\Xi\subseteq\mathbb{R}^m$ such that $\mathbb{E}^{\xi\sim\mathbb{Q}}\{\|\xi\|\}<\infty$ for some norm $\|\cdot\|$ on $\mathbb{R}^m$. When it is evident from the context, we replace $\mathbb{E}^{\xi\sim\mathbb{Q}}\{\cdot\}$ with $\mathbb{E}^{\mathbb{Q}}\{\cdot\}$. The Wasserstein distance $\mathfrak{W}:\mathcal{M}(\Xi)\times \mathcal{M}(\Xi)\rightarrow \mathbb{R}_{\geq 0}:=\{x\in \mathbb{R}|x\geq 0\}$ is defined as
\begin{align*}
\mathfrak{W}(\mathbb{P}_1,\mathbb{P}_2):=&\inf \Bigg\{\int_{\Xi^2} \|\xi_1-\xi_2\|\Pi(\mathrm{d}\xi_1,\mathrm{d}\xi_2): \\
& \Pi \mbox{ is a joint disribution on } \xi_1 \mbox{ and }\xi_2\\
& \mbox{with marginals }\mathbb{P}_1\mbox{ and }\mathbb{P}_2\mbox{, respectively}\Bigg\},
\end{align*}
for all $\mathbb{P}_1,\mathbb{P}_2\in\mathcal{M}(\Xi)$~\cite{kantorovich1958space}. The decision variable $\Pi$ can be interpreted as a transportation plan for moving mass distribution described by $\mathbb{P}_1$ to mass distribution described by $\mathbb{P}_2$. In this case, $\mathfrak{W}(\mathbb{P}_1,\mathbb{P}_2)$ can be seen as the optimal mass transportation plan. This metric is sometimes referred to as the Earth-Mover  distance. Note that $\mathfrak{W}(\mathbb{P}_1,\mathbb{P}_2)=\mathfrak{W}(\mathbb{P}_2,\mathbb{P}_1)$. We can compute the Wasserstein distance using an alternative approach due to~\cite{kantorovich1958space} by 
\begin{align*}
\mathfrak{W}(\mathbb{P}_1,\mathbb{P}_2):=\sup_{f\in\mathbb{L}} \Bigg\{&\hspace{-.03in}\int_{\Xi} f(\xi)\mathbb{P}_1(\mathrm{d}\xi)-\int_{\Xi} f(\xi)\mathbb{P}_2(\mathrm{d}\xi)\Bigg\},
\end{align*}
where $\mathbb{L}$ denotes the set of all Lipschitz functions with Lipschitz constant upper bounded by one, i.e., all functions $f$ such that $|f(\xi_1)-f(\xi_2)|\leq \|\xi_1-\xi_2\|$ for all $\xi_1,\xi_2\in\Xi$. Now, we are ready to pose the problem of training machine learning models under poisoning attacks using distributionally-robust optimization.

\section{Distributionally-Robust Machine Learning \\ Against Poisoning Attacks}
\label{sec:robustML}
Consider a supervised learning problem with dataset $\{(x_i,y_i)\}_{i=1}^n$, where $x_i\in\mathbb{R}^{p_x}$ denotes the inputs and $y_i\in\mathbb{R}^{p_y}$ denotes the outputs. For instance, in image classification, $x_i$ is a vector of pixels from an image (which could be compressed by extracting important features from the image) and $y_i$ is the class to which that image belongs (mapped to a set of integers). Our goal is to extract a machine learning model $\mathfrak{M}(\cdot;\theta):\mathbb{R}^{p_x} \rightarrow\mathbb{R}^{p_y}$ to describe the relationship between inputs and outputs. This is often done by solving the stochastic program
\begin{align} \label{eqn:optimization:1}
J^*:=\inf_{\theta\in\Theta}  \mathbb{E}^{\mathbb{P}} \{\ell(\mathfrak{M}(x;\theta),y)\},
\end{align}
where $\Theta$ is the set of feasible model parameters, $\ell:\mathbb{R}^{p_y}\times \mathbb{R}^{p_y}\rightarrow \mathbb{R}$ is the loss function, and $\mathbb{P}$ is the distribution of the inputs and outputs $(x,y)$ of the model. Since we do not know the distribution $\mathbb{P}$, we rely on the training dataset  $\{(x_i,y_i)\}_{i=1}^n$ to solve the sample-averaged approximation problem in 
\begin{align} \label{eqn:optimization:2}
\hat{J}:=\inf_{\theta\in\Theta}  \frac{1}{n} \sum_{i=1}^n\ell(\mathfrak{M}(x_i;\theta),y_i).
\end{align}
Even if we knew $\mathbb{P}$, we cannot solve~\eqref{eqn:optimization:1} directly (in general) as integrating the loss function for arbitrary distributions is a computationally complicated task~\cite{hanasusanto2016comment}. The approximation in~\eqref{eqn:optimization:2} is the basis of machine learning literature. When $n$ is large enough, solving~\eqref{eqn:optimization:2} is a good proxy for~\eqref{eqn:optimization:1} if the training dataset is clean. However, in practice, the dataset might be compromised due to adversarial inputs. 

We assume that the training dataset is composed of two types of correct and malicious samples. The correct samples are independently and identically distributed according to probability density function $\mathbb{P}$. The malicious samples are inserted into the training dataset by an adversary with the hope manipulating the trained model. These samples are  independently and identically distributed according to probability density function $\mathbb{Q}$. \textit{When training the machine learning model, we do not know either of these probability distributions.} We only know that a sample is correct with probability $1-\beta$ and malicious with probability $\beta$. That is, the ratio of the corrupted data entries is $\beta$.  
Hence, any given record  $(x_i,y_i)$ in the training dataset are independently and identically distributed according to $\mathbb{D}=(1-\beta)\mathbb{P}+\beta \mathbb{Q}$. We make the following standing assumption regarding the distribution of the training data. 

\begin{assumption} \label{assum:1} There exists constant $a>1$ such that $\mathbb{E}^{\mathbb{D}}\{\exp(\|\xi\|^a)\}<\infty$.
\end{assumption}

Assumption~\ref{assum:1} implies that $\mathbb{D}$, which is composed of $\mathbb{P}$ and $\mathbb{Q}$, is a light-tailed distribution. All probability distributions with compact (i.e., bounded and closed) support meet this condition. This assumption is often implicit in the machine learning literature as, for heavy-tailed distributions (i.e., distributions that do not meet Assumption~\ref{assum:1}), the sample average of the loss may not even converge to the expected loss in general~\cite{brownlees2015empirical,catoni2012challenging} and, as a result,~\eqref{eqn:optimization:2} might not be a good proxy for~\eqref{eqn:optimization:1}.

Let us define the empirical probability distribution
\begin{align*}
\widehat{\mathbb{D}}_n:=\frac{1}{n}\sum_{i=1}^n \delta_{(x_i,y_i)},
\end{align*}
where $\delta_\xi$ is the Dirac distribution function with its mass concentrated at $\xi$, i.e., $\int \delta_\xi(x)\mathrm{d}x=1$ and $\delta_\xi(x)=0$ for all $x\neq \xi$. Evidently,
\begin{align}
\frac{1}{n} \sum_{i=1}^n\ell(\mathfrak{M}(x_i;\theta),y_i)=\mathbb{E}^{\widehat{\mathbb{D}}_n} \{\ell(\mathfrak{M}(x;\theta),y)\},
\end{align}
and, as a result,~\eqref{eqn:optimization:2} can be rewritten as 
\begin{align} 
\hat{J}:=\inf_{\theta\in\Theta}  \mathbb{E}^{\widehat{\mathbb{D}}_n} \{\ell(\mathfrak{M}(x;\theta),y)\}.
\end{align}
By distributionally-robust machine learning, we mean that we 
would like to devise an algorithm for extracting a trained machine learning model  $\hat{\theta}_n\in\Theta$ such that 
\begin{align} \label{eqn:guarantee}
\mathbb{D}^n\{J^*\leq \mathbb{E}^{\mathbb{P}}\{\ell(\mathfrak{M}(x;\hat{\theta}_n),y)\}\leq \hat{J}_n\}\geq 1-\beta,
\end{align}
where $\beta\in(0,1)$ is a significance parameter with respect to distribution $\mathbb{D}^n$ and $\hat{J}_n$ is an upper bound that may depend on the training dataset $(x_i,y_i)_{i=1}^n$ and the trained model $\hat{\theta}_n$. Clearly, $\hat{J}_n$ and $\hat{\theta}_n$ both depend on the choice of significance parameter $\beta\in(0,1)$. 

In this paper, we extract the trained model by solving the  distributionally-robust optimization problem in
\begin{align} \label{eqn:optimization:3}
\hat{J}_n:=\inf_{\theta\in\Theta} \sup_{\mathbb{G}:\mathfrak{W}(\mathbb{G},\widehat{\mathbb{D}}_n)\leq \rho}  \mathbb{E}^{\mathbb{G}} \{\ell(\mathfrak{M}(x;\theta),y)\},
\end{align}
for some constant $\rho>0$. Note that $\hat{\theta}_n\in\Theta$ is the minimizer of~\eqref{eqn:optimization:3}. This problem formulation is motivated by that the distributions of the correct samples $\mathbb{P}$ is within a small neighbourhood of the empirical probability distribution $\widehat{\mathbb{D}}_n$ in the sense of the the Wasserstein distance if $n$ is large enough and $\beta$ is relatively small. Therefore, optimizing $\sup_{\mathbb{G}:\mathfrak{W}(\mathbb{G},\widehat{\mathbb{D}}_n)\leq \rho}  \mathbb{E}^{\mathbb{G}} \{\ell(\mathfrak{M}(x;\theta),y)\}$ can act as a surrogate for optimizing $\mathbb{E}^{\mathbb{P}} \{\ell(\mathfrak{M}(x;\theta),y)\}$. This is proved in the following lemma. 

\begin{lemma} \label{lemma:1} There exist constants $c_1,c_2>0$ such that 
	\begin{align*}
	\mathbb{D}^n\{\mathfrak{W}(\widehat{\mathbb{D}}_n,\mathbb{P})
	\leq \zeta(\gamma)+\beta\mathfrak{W}(\mathbb{Q},\mathbb{P})\}
	\geq 
	1- \gamma,
	\end{align*}
	where
	\begin{align*}
	\zeta(\gamma):=\begin{cases}
	\displaystyle\left(\frac{\log(c_1/\gamma)}{c_2n}\right)^{1/\max\{p,2\}}, & \displaystyle n\geq \frac{\log(c_1/\gamma)}{c_2}, \\
	\displaystyle\left(\frac{\log(c_1/\gamma)}{c_2n}\right)^{1/a}, & \displaystyle n< \frac{\log(c_1/\gamma)}{c_2},
	\end{cases}
	\end{align*}
	for all $n\geq 1$, $p=p_x+p_y\neq 2$, and $\gamma>0$.
\end{lemma}

\begin{proof}
Following~\cite{fournier2015rate} and~\cite{esfahani2018data}, we know that 
$\mathbb{D}^n\{\mathfrak{W}(\widehat{\mathbb{D}}_n,\mathbb{D})
\leq \zeta(\gamma)\}
\leq 1-\gamma.$
Now, note that
\begin{align}
\mathfrak{W}(\mathbb{D},\mathbb{P})
&=\mathfrak{W}((1-\beta)\mathbb{P}+\beta \mathbb{Q},\mathbb{P})\nonumber\\
&\leq (1-\beta)\mathfrak{W}(\mathbb{P},\mathbb{P})+\beta \mathfrak{W}(\mathbb{Q},\mathbb{P})\nonumber\\
&\leq \beta \mathfrak{W}(\mathbb{Q},\mathbb{P}),\label{eqn:proof:1}
\end{align}
where the first inequality follows from the convexity of the Wasserstein distance~\cite[Lemma~2.1]{pflug2014multistage} and the second inequality follows from that $\mathfrak{W}(\mathbb{P},\mathbb{P})=0$. The triangle inequality for the Wasserstein distance~\cite[p.\,170]{prugel2020probability} implies that
$\mathfrak{W}(\widehat{\mathbb{D}}_n,\mathbb{P})
\leq \mathfrak{W}(\widehat{\mathbb{D}}_n,\mathbb{D})+\mathfrak{W}(\mathbb{D},\mathbb{P}),$
and, using the inequality in~\eqref{eqn:proof:1}, $
\mathfrak{W}(\widehat{\mathbb{D}}_n,\mathbb{P})
\leq \mathfrak{W}(\widehat{\mathbb{D}}_n,\mathbb{D})+\beta \mathfrak{W}(\mathbb{Q},\mathbb{P}).$
Therefore, $\mathfrak{W}(\widehat{\mathbb{D}}_n,\mathbb{P})
\leq \zeta(\gamma)+\beta\mathfrak{W}(\mathbb{D},\mathbb{P})$ if 
$\mathfrak{W}(\widehat{\mathbb{D}}_n,\mathbb{D})
\leq \zeta(\gamma)$, which implies that $
\mathbb{D}^n\{\mathfrak{W}(\widehat{\mathbb{D}}_n,\mathbb{P})
\leq \zeta(\gamma)+\beta\mathfrak{W}(\mathbb{Q},\mathbb{P})\}
\geq 
1- \gamma.$
This concludes the proof. 
\end{proof}

Now, that we know that the distributions of the correct samples $\mathbb{P}$ is within a small neighbourhood of the empirical probability distribution $\widehat{\mathbb{D}}_n$ (defined by the Wasserstein distance) with a high probability, the solution of~\eqref{eqn:optimization:3} should provide a bound for the expected loss function in the sense of~\eqref{eqn:guarantee}. This is proved in the following theorem. 

\begin{theorem}
Assume that  $\rho=\zeta(\beta) +\beta\mathfrak{W}(\mathbb{Q},\mathbb{P})$ in~\eqref{eqn:optimization:3}. Then,~\eqref{eqn:guarantee} holds. 
\end{theorem}

\begin{proof}
The proof follows from~\cite[Theorem 3.4]{esfahani2018data}. 
\end{proof}

We can solve~\eqref{eqn:optimization:3} explicitly following the approach in~\cite{esfahani2018data}. However, in this paper, we show that we can simplify the aforementioned distributionally-robust optimization problem by finding an upper bound for the worst-case expected fitness  $\sup_{\mathbb{G}:\mathfrak{W}(\mathbb{G},\widehat{\mathbb{D}}_n)\leq \rho}  \mathbb{E}^{\mathbb{G}} \{\ell(\mathfrak{M}(x;\theta),y)\}$. This is done in the following lemma. 

\begin{lemma} \label{lemma:2} Assume that $\ell(\mathfrak{M}(x;\theta),y)$ is $L(\theta)$-Lipschitz continuous in $(x,y)$ for a fixed $\theta\in\Theta$. Then, 
\begin{align*}
\sup_{\mathbb{G}:\mathfrak{W}(\mathbb{G},\widehat{\mathbb{D}}_n)\leq \rho}  \mathbb{E}^{\mathbb{G}} \{\ell(&\mathfrak{M}(x;\theta),y)\}\leq \mathbb{E}^{\widehat{\mathbb{D}}_n} \{\ell(\mathfrak{M}(x;\theta),y)\}
+L(\theta)\rho.
\end{align*}
\end{lemma}

\begin{proof}First, note that
\begin{align*}
\mathbb{E}^{\mathbb{G}} \{\ell(\mathfrak{M}(x;\theta),y)\}
=&\mathbb{E}^{\widehat{\mathbb{D}}_n}\{\ell(\mathfrak{M}(x;\theta),y)
\}\\&-\mathbb{E}^{\widehat{\mathbb{D}}_n}\{\ell(\mathfrak{M}(x;\theta),y)\}+\mathbb{E}^{\mathbb{G}} \{\ell(\mathfrak{M}(x;\theta),y)\}\\
\leq & \mathbb{E}^{\widehat{\mathbb{D}}_n}\{\ell(\mathfrak{M}(x;\theta),y)\}
\\&+|\mathbb{E}^{\mathbb{G}} \{\ell(\mathfrak{M}(x;\theta),y)\}-\mathbb{E}^{\widehat{\mathbb{D}}_n}\{\ell(\mathfrak{M}(x;\theta),y)|.
\end{align*}
We have
\begin{align*}
|\mathbb{E}^{\mathbb{G}}\{\ell(\mathfrak{M}(x;\theta),y)\}
-&\mathbb{E}^{\widehat{\mathbb{D}}_n}\{\ell(\mathfrak{M}(x;\theta),y)\}|\\
= &\Bigg|\int_{x,y} \ell(\mathfrak{M}(x;\theta),y)\mathbb{G}(\mathrm{d}x,\mathrm{d}y) \\
&-\int_{x',y'} \ell(\mathfrak{M}(x';\theta),y')\widehat{\mathbb{D}}_n(\mathrm{d}x',\mathrm{d}y') \Bigg|\\
= &\Bigg|\int_{x,y} \ell(\mathfrak{M}(x;\theta),y)\int_{x',y'}\Pi(\mathrm{d}x,\mathrm{d}y,\mathrm{d}x',\mathrm{d}y') \\
&-\int_{x',y'} \ell(\mathfrak{M}(x';\theta),y')\int_{x,y}\Pi(\mathrm{d}x,\mathrm{d}y,\mathrm{d}x',\mathrm{d}y') \Bigg|\\
= &\Bigg|\int_{x,y,x',y'} (\ell(\mathfrak{M}(x;\theta),y)-\ell(\mathfrak{M}(x';\theta),y'))\\
&\hspace{1in}\times \Pi(\mathrm{d}x,\mathrm{d}y,\mathrm{d}x',\mathrm{d}y') \Bigg|,
\end{align*}
where $\Pi$ is a joint disribution on $(x,y)$ and $(x',y')$ with marginals $\mathbb{G}$ and $\widehat{\mathbb{D}}_n$, respectively. Therefore, 
\begin{align*}
\Big|\mathbb{E}^{\mathbb{G}}\{\ell(\mathfrak{M}(x;\theta),y)\}
-\mathbb{E}^{\widehat{\mathbb{D}}_n}&\{\ell(\mathfrak{M}(x;\theta),y)\}\Big|\\
\leq & \int_{x,y,x',y'} |\ell(\mathfrak{M}(x;\theta),y)-\ell(\mathfrak{M}(x';\theta),y')|\\
&\hspace{1in}\times \Pi(\mathrm{d}x,\mathrm{d}y,\mathrm{d}x',\mathrm{d}y')\\
\leq & \int_{x,y,x',y'} \frac{|\ell(\mathfrak{M}(x;\theta),y)-\ell(\mathfrak{M}(x';\theta),y')|}{\|(x,y)-(x',y')\|}\\
&\hspace{.3in}\times\|(x,y)-(x',y')\| \Pi(\mathrm{d}x,\mathrm{d}y,\mathrm{d}x',\mathrm{d}y')\\
\leq & L(\theta)\int_{x,y,x',y'} \hspace{-.2in}\|(x,y)-(x',y')\| \Pi(\mathrm{d}x,\mathrm{d}y,\mathrm{d}x',\mathrm{d}y')\\
\leq & L(\theta)\rho.
\end{align*}
This concludes the proof.
\end{proof}

For continuously-differentiable $\ell(\mathfrak{M}(x;\theta),y)$, we can compute the Lipschitz constant based on the gradients. However, continuous differentiability is not required for Lipschitz continuity~\cite[p.\,61]{reddy1998introductory}. We explore continuously-differentiable loss functions in the next lemma. 

\begin{lemma} \label{lemma:3} $\ell(\mathfrak{M}(x;\theta),y)$ is $L(\theta)$-Lipschitz continuous in $(x,y)$ with 
	\begin{align*}
	L(\theta):=\sup_{(x,y)\in\supp(\mathbb{D})}\|\nabla_{(x,y)}\ell(\mathfrak{M}(x;\theta),y)\|_*.
	\end{align*}
\end{lemma}

\begin{proof}
If $\ell(\mathfrak{M}(x;\theta),y)$ is continuously differentiable in $(x,y)$, the mean value theorem~\cite{doi1010800025570X197911976771} implies that, there exists $(x'',y'')$ over the convex combination of $(x,y)$ and $(x',y')$ such that 
\begin{align*}
|\ell(\mathfrak{M}(x;\theta),y)-\ell(\mathfrak{M}(x';\theta),y')|
&
\leq \Big|\nabla_{(x,y)}\ell(\mathfrak{M}(x'';\theta),y'')^\top((x,y)-(x',y')) \Big|\\
&
\leq 
\|\nabla_{(x,y)}\ell(\mathfrak{M}(x'';\theta),y'')\|_* \|(x,y)-(x',y')\|,
\end{align*}
where $\|\cdot\|_*$ is the dual norm of $\|\cdot\|$. Hence, we can show that $\ell(\mathfrak{M}(x;\theta),y)$ is $L(\theta)$-Lipschitz continuous in $(x,y)$.
\end{proof}

Note that we can select any function of the model parameters $L(\theta)$ that is larger than the norm of the gradient  $\sup_{(x,y)\in\supp(\mathbb{D})}\|\nabla_{(x,y)}\ell(\mathfrak{M}(x;\theta),y)\|_*$ in Lemma~\ref{lemma:3} as the Lipschitz constant, and thus as the regularizer. The choice in Lemma~\ref{lemma:3} is merely the tightest Lipschitz constant for continuously-differentiable loss functions. 

Now, we can define the regularized sample-averaged optimization problem in 
\begin{align} \label{eqn:optimization:4}
\hat{J}_n:=\inf_{\theta\in\Theta}  \Big[\mathbb{E}^{\widehat{\mathbb{D}}_n} \{\ell(\mathfrak{M}(x;\theta),y)\}+\rho L(\theta)\Big].
\end{align}
We can similarly define $\hat{\theta}_n\in\Theta$ as the minimizer of~\eqref{eqn:optimization:4}.  
This optimization problem no longer involves taking a supremum over the probability density function and is thus computationally much easier to solve in compassion to~\eqref{eqn:optimization:3} (as it does not require solving an infinite-dimensional optimization problem). We can still prove a performance guarantee for the optimizer of~\eqref{eqn:optimization:4} in the sense of~\eqref{eqn:guarantee}. This is done in the next theorem. 

\begin{theorem}
	Assume that  $\rho=\zeta(\beta) +\beta\mathfrak{W}(\mathbb{Q},\mathbb{P})$ in~\eqref{eqn:optimization:4}. Then,~\eqref{eqn:guarantee} holds. 
\end{theorem}

\begin{proof}
The proof is similar to~\cite[Theorem 3.4]{esfahani2018data} with an extra step with the aid of Lemma~\ref{lemma:2}. By selecting $\rho=\zeta(\beta) +\beta\mathfrak{W}(\mathbb{Q},\mathbb{P})$, according to Lemma~\ref{lemma:1}, $\mathbb{P}$ belongs to a ball around $\widehat{\mathbb{D}}_n$ with radius $\rho$ with probability greater than or equal to $1-\beta$. Therefore, with probability of at least $1-\beta$, $\mathbb{E}^{\mathbb{P}} \{\ell(\mathfrak{M}(x;\hat{\theta}_n),y)\}\leq \sup_{\mathbb{G}:\mathfrak{W}(\mathbb{G}, \widehat{\mathbb{D}}_n)\leq \rho} \mathbb{E}^{\mathbb{G}} \{\ell(\mathfrak{M}(x;\hat{\theta}_n),y)\}$. Further, Lemma~\ref{lemma:2} states that  $\sup_{\mathbb{G}:\mathfrak{W}(\mathbb{G}, \widehat{\mathbb{D}}_n)\leq \rho} \linebreak \mathbb{E}^{\mathbb{G}} \{\ell(\mathfrak{M}(x;\hat{\theta}_n),y)\}\leq \hat{J}_n$. Therefore, with probability of at least $1-\beta$,  we have $\mathbb{E}^{\mathbb{P}} \{\ell(\mathfrak{M}(x;\hat{\theta}_n),y)\}\leq\hat{J}_n$. This concludes the proof. 
\end{proof}

Although the results of this paper hold for any machine learning problem, convex or non-convex, we simplify the regularizer for linear and logistic regression models for demonstration purposes. This is the topic of the next section.

\section{Application to Regression}
\label{sec:regression}
We start by investigating linear regression and then continue to logistic regression. 

\subsection{Linear Regression}
We start with linear regression. Without loss of generality, assume $p_y=1$. If  $p_y>1$, we can treat each output independently.  In this case, 
\begin{align*}
\mathfrak{M}(x;\theta)=\theta^\top 
\begin{bmatrix}
x \\
1
\end{bmatrix},
\end{align*}
and the fitness function is
\begin{align} \label{eqn:quadratic}
\ell(\mathfrak{M}(x;\theta),y)=\frac{1}{2}(\mathfrak{M}(x;\theta)-y)^2.
\end{align}
We also assume that $(x,y)$ belong to compact set $\mathcal{X}\times\mathcal{Y}\subseteq\mathbb{R}^{p_x}\times\mathbb{R}$.
Following Lemma~\ref{lemma:3}, we have
\begin{align} \label{eqn:linear_regression_L}
L(\theta)=\max_{(x,y)\in \mathcal{X}\times\mathcal{Y}}\left|\theta^\top 
\begin{bmatrix}
x \\
1
\end{bmatrix}-y\right| \left\|W\theta \right\|_*,
\end{align}
where 
\begin{align*}
W:=\begin{bmatrix}
I_{p_x} & 0_{p_x\times 1} \\
0_{1 \times p_x}  & 0
\end{bmatrix}.
\end{align*}
The regularizer in~\eqref{eqn:linear_regression_L} is the tightest that we can use in~\eqref{eqn:optimization:4}. However, it requires solving an optimization problem for computing the regularizer which could lead to  the use of additional computational resources. 
We can further simplify this Lipschitz constant by noting that
\begin{align*}
\max_{(x,y)\in \mathcal{X}\times\mathcal{Y}}\left|\theta^\top 
\begin{bmatrix}
x \\
1
\end{bmatrix}\hspace{-.03in}-\hspace{-.02in}y\right|
\leq &
\max_{(x,y)\in \mathcal{X}\times\mathcal{Y}}\left|\theta^\top 
\begin{bmatrix}
x \\
1
\end{bmatrix}\right|+|y|\\
\leq &
\max_{(x,y)\in \mathcal{X}\times\mathcal{Y}}\|\theta\|_*(\|x\|\hspace{-.02in}+\hspace{-.02in}1)
\hspace{-.02in}+\hspace{-.02in}|y|,
\end{align*}
and
\begin{align*}
\left\|W\theta \right\|_*
\leq \|\theta\|_*.
\end{align*}
Therefor, we get a more conservative (i.e., larger) Lipschitz constant:
\begin{align*}
L(\theta)=(X+1)\|\theta\|_*^2+Y \|\theta \|_*,
\end{align*}
where $X=\max_{x\in \mathcal{X}}\|x\|$ and $Y=\max_{y\in\mathcal{Y}}|y|$. 
Therefore, for linear regression, the adversarially-robust optimization problem in~\eqref{eqn:optimization:4} can be rewritten as 
\begin{align} \label{eqn:optimization:4_regression_quadratic}
\hat{J}_n:=\inf_{\theta\in\Theta}  \Bigg[ \frac{1}{n} \sum_{i=1}^n&\ell(\mathfrak{M}(x_i;\theta),y_i)+\rho (X+1)\|\theta\|_*^2+\rho Y \|\theta \|_*\Bigg].
\end{align}
An alternative to the quadratic loss function in~\eqref{eqn:quadratic} is 
\begin{align}
\ell(\mathfrak{M}(x;\theta),y)=|\mathfrak{M}(x;\theta)-y|.
\end{align}
In this case, we get
$
L(\theta)=\left\|\theta \right\|_*$ and the optimization problem in~\eqref{eqn:optimization:4} can be rewritten as 
\begin{align} \label{eqn:optimization:4_regression_linear}
\hat{J}_n:=\inf_{\theta\in\Theta}  \Bigg[ &\frac{1}{n} \sum_{i=1}^n\ell(\mathfrak{M}(x_i;\theta),y_i)+\rho \|\theta\|_*^2\Bigg].
\end{align}
Now, we move on to the logistic regression.

\subsection{Logistic Regression}
In this case, we have
\begin{align*}
\mathfrak{M}(x;\theta)=\frac{1}{1+\exp(-[x^\top \; 1]\theta)},
\end{align*}
and the fitness function is
\begin{align}
\ell(\mathfrak{M}(x;\theta),y)=&-y\log(\mathfrak{M}(x;\theta))\nonumber
\\
&-(1-y)\log(1-\mathfrak{M}(x;\theta)).
\end{align}
Note the opposite sign with the log likelihood as here we minimize the expected fitness. We have
\begin{align*}
\nabla_y \ell(\mathfrak{M}(x;\theta),y)
=&-\log(\mathfrak{M}(x;\theta))-\log(1-\mathfrak{M}(x;\theta))\\
=&-\log\left(\frac{\mathfrak{M}(x;\theta)}{1-\mathfrak{M}(x;\theta)}\right)\\
=&-[x^\top \; 1]\theta,
\end{align*}
and
\begin{align*}
\nabla_x \ell(\mathfrak{M}(x;\theta),y)
=&-\Bigg(y -\frac{1}{1+\exp(-[x^\top \; 1]\theta)}\Bigg)W\theta.
\end{align*}
Therefore, we get
\begin{align*}
L(\theta)=(Y+X+2)\|\theta\|_*.
\end{align*}
The robust optimization problem in~\eqref{eqn:optimization:4} becomes
\begin{align} \label{eqn:optimization:4_logistic_regression}
\hat{J}_n:=\inf_{\theta\in\Theta}  \Bigg[ \frac{1}{n} \sum_{i=1}^n&\ell(\mathfrak{M}(x_i;\theta),y_i)+\rho (Y+X+2)\|\theta\|_*\Bigg]\hspace{-.02in}.
\end{align}
In the remainder of this paper, we demonstrate the results of this paper for regression models discussed above.

\begin{figure}[t]
	\centering 
	\begin{tikzpicture}
	\node[] at (0,0) {\includegraphics[width=.65\linewidth]{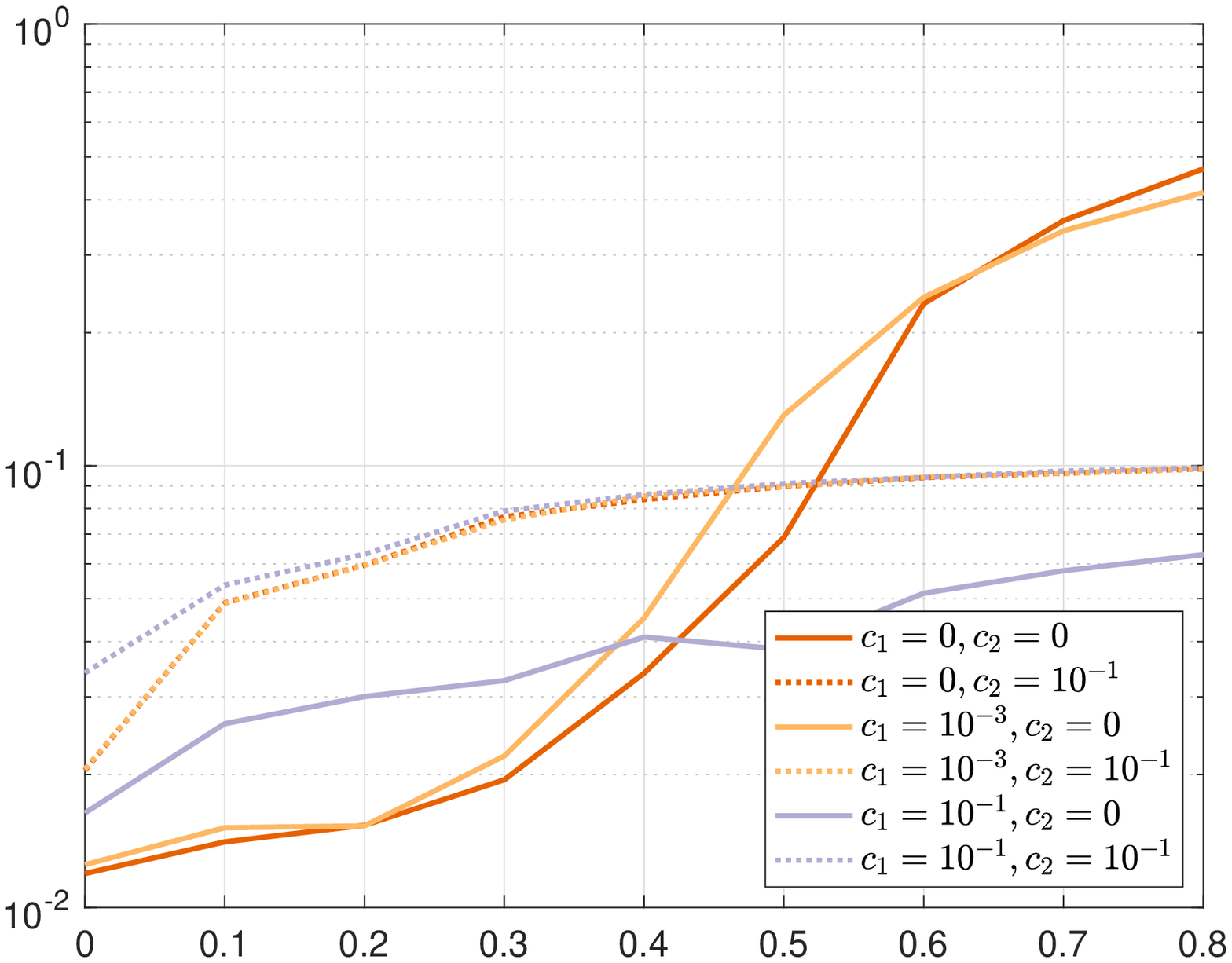}};
	\node[] at (0,-3.1) {$\beta$};
	\node[rotate=90] at (-3.9,0) {$\mathbb{E}^{\mathbb{P}}\{\ell(\mathfrak{M}(x;\theta),y)\}$};
	\end{tikzpicture}
	\caption{\label{fig:1} Test performance of the trained model for the Wine Quality dataset when using~\eqref{eqn:optimization:4_regression_quadratic}  to mitigate the effect of data modification attacks versus the percentage of the poisoned data~$\beta$ for various choices of regularization.}
\end{figure}

\begin{figure}[t]
	\centering 
	\begin{tikzpicture}
	\node[] at (0,0) {\includegraphics[width=.65\linewidth]{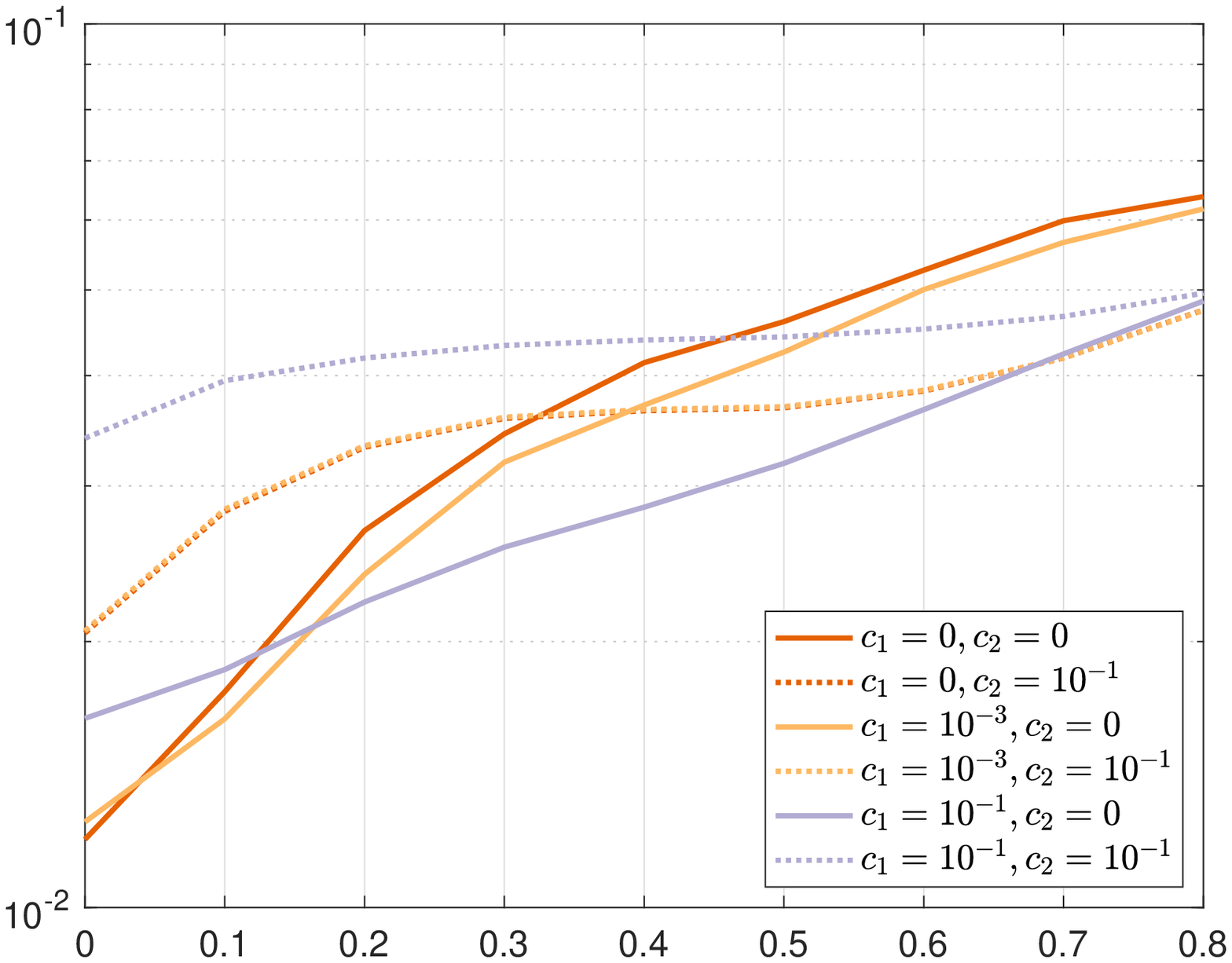}};
	\node[] at (0,-3.1) {$\beta$};
	\node[rotate=90] at (-3.9,0) {$\mathbb{E}^{\mathbb{P}}\{\ell(\mathfrak{M}(x;\theta),y)\}$};
	\end{tikzpicture}
	\caption{\label{fig:2}Test performance of the trained model for the Wine Quality dataset when using~\eqref{eqn:optimization:4_regression_quadratic}  to mitigate the effect of label flipping attacks versus the percentage of the poisoned data~$\beta$ for various choices of regularization.}
\end{figure}

\begin{figure}[t]
	\centering
	\begin{tikzpicture}
	\node[] at (0,0) {\includegraphics[width=.65\linewidth]{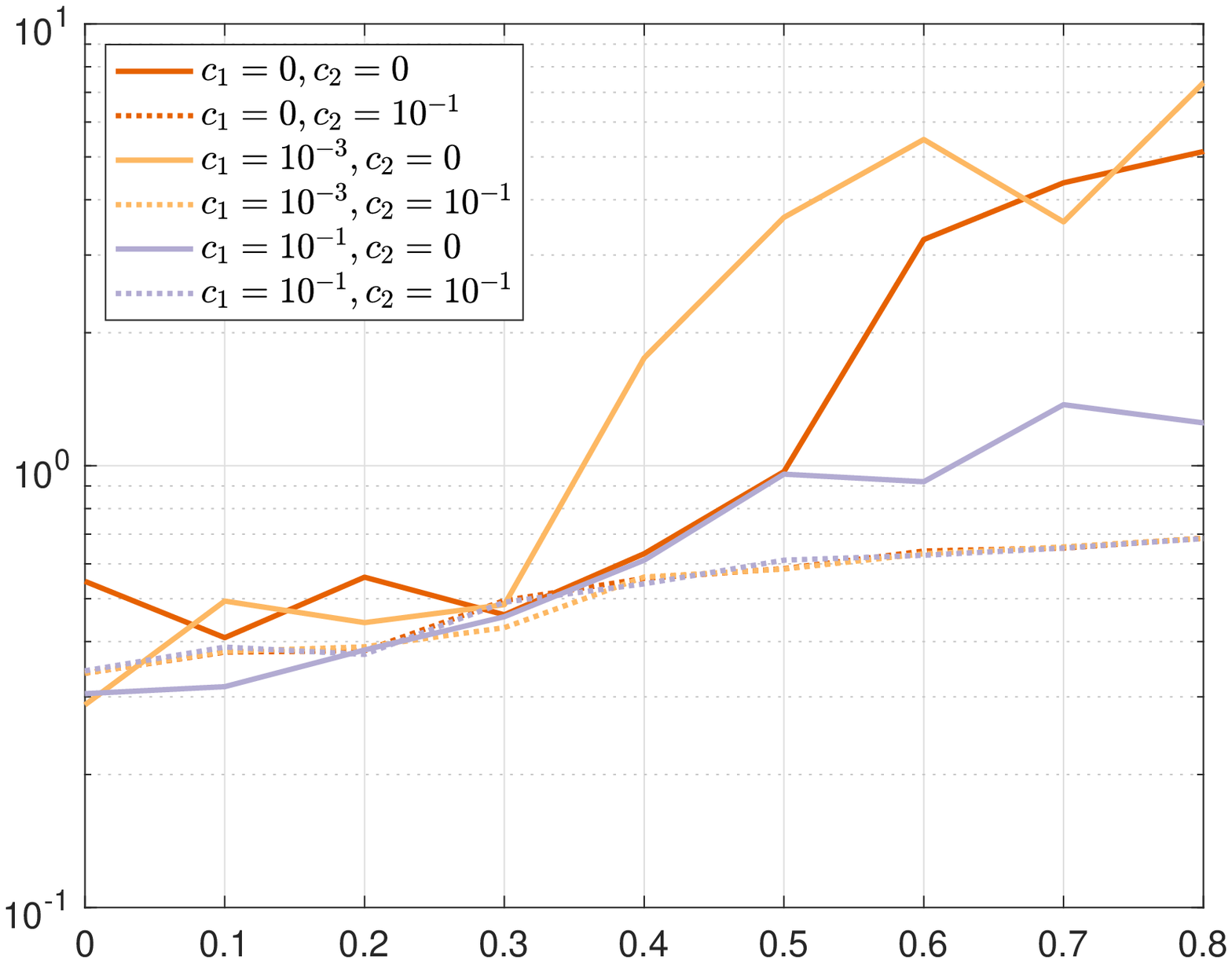}};
	\node[] at (0,-3.1) {$\beta$};
	\node[rotate=90] at (-3.9,0) {$\mathbb{E}^{\mathbb{P}}\{\ell(\mathfrak{M}(x;\theta),y)\}$};
	\end{tikzpicture}
	\caption{\label{fig:3} Test performance of the trained model for the Boston Housing Market  dataset when using~\eqref{eqn:optimization:4_regression_quadratic}  to mitigate the effect of data modification attacks versus the percentage of the poisoned data~$\beta$ for various choices of regularization.}
\end{figure}

\begin{figure}[t]
		\centering
	\begin{tikzpicture}
	\node[] at (0,0) {\includegraphics[width=.65\linewidth]{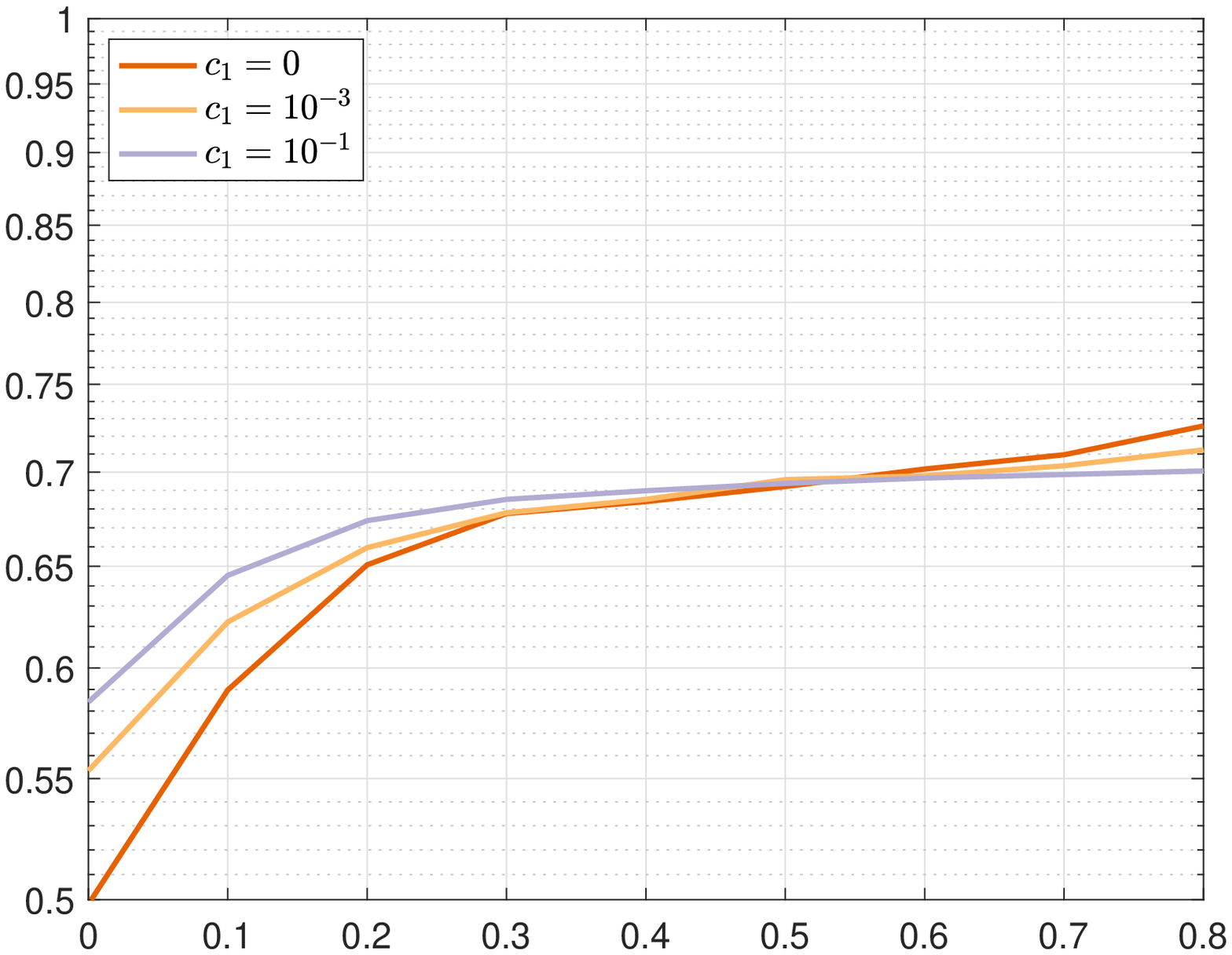}};
	\node[] at (0,-3.1) {$\beta$};
	\node[rotate=90] at (-3.9,0) {$\mathbb{E}^{\mathbb{P}}\{\ell(\mathfrak{M}(x;\theta),y)\}$};
	\end{tikzpicture}
	\caption{\label{fig:4}Test performance of the trained model for the Adult dataset when using~\eqref{eqn:optimization:4_logistic_regression}  to mitigate the effect of label flipping attacks versus the percentage of the poisoned data~$\beta$ for various choices of regularization.}
\end{figure}

\section{Numerical Example}
\label{sec:example}

\subsection{Dataset Description}
We use three different datasets to demonstrate the results of this paper. The first dataset is the red Wine Quality dataset containing 1599 records of red wine tests~\cite{cortez2009modeling}. There are 11 inputs capturing physicochemical and sensory measurements for the wine. The output is the quality scores ranging from zero to ten.  The second dataset is the Boston Housing Market dataset containing 506 records from Boston neighbourhood house prices~\cite{harrison1978hedonic}. The inputs are features, such as per capita crime rate in the neighbourhood and average number of rooms. The output is the median home prices in the neighbourhood. The final dataset is the Adult dataset containing 48,842 records from the 1994 Census~\cite{kohavi1996scaling}. In this dataset, there are 14 individual attributes, such as age, race, occupation, and relationship status, as inputs and income level (i.e., above or below \$50K per annum) as output. We translate all categorical attributes and outputs to integers.

\subsection{Experiment Setup}
We demonstrate the effect of data poisoning on regression. We train linear regression models for the  Wine Quality and the Boston Housing Market datasets according to~\eqref{eqn:optimization:4_regression_quadratic} and train logistic regression models for the Adult dataset according to~\eqref{eqn:optimization:4_logistic_regression}. We split each dataset into two equal halves: one for training and the other for test. We consider two form of attacks: label flipping and data modification.  

When using linear regression for the  Wine Quality dataset, we use two attacks. The first attack is the data modification by changing the features of $\beta$ percent of the dataset to a Gaussian variable with variance of 4. The second attack is a label flipping the output for $\beta$ percent of the dataset to 10. When using linear regression for the Boston Housing Market dataset, we use the data modification by changing the features of $\beta$ percent of the dataset to a Gaussian variable with variance of 100. Finally, when using logistic regression for the Adult dataset, we perform label flipping  attack by negating the output for $\beta$ percent of the dataset. 

\subsection{Experimental Results}
We first consider the Wine Quality dataset. Figure~\ref{fig:1} shows the test performance of the trained model using~\eqref{eqn:optimization:4_regression_quadratic}  for the data modification attack versus the percentage of the poisoned data $\beta$ for various choices of regularization. In this graph, $c_1$ denotes the weight behind $\|\theta\|_*^2$ and $c_2$ denotes the weight behind $\|\theta\|_*$ in~\eqref{eqn:optimization:4_regression_quadratic} . Here, we use the 2-norm as $\|\cdot\|$. Therefore, $\|\cdot\|_*$ is also equal to the 2-norm. For large enough attack levels (i.e., large $\beta$), the test performance for the regularized model is significantly better. However, this comes at the cost of a lower performance when there is no poisoning attack. This is because we are solving a distributionally-robust machine learning problem to mitigate the effect of poisoning attacks (with the conservative-ness dictated by $\rho$ which is proportional to $c_1$ and $c_2$). The robust model can act more conservative if there are no attacks as in those cases there was no need for robustness and regularization. Figure~\ref{fig:2} shows the test performance of the trained model using~\eqref{eqn:optimization:4_regression_quadratic}  for the label flipping attack versus the percentage of the poisoned data $\beta$ for various choices of regularization. The same trend can also be observed in this case. 

Now, we move to the Boston Housing Market dataset. Figure~\ref{fig:3} shows the test performance of the trained model using~\eqref{eqn:optimization:4_regression_quadratic}  for the data modification attack versus the percentage of the poisoned data $\beta$ for various choices of regularization. The regularized models beat the non-regularized one in terms of performance in the presence of poisoning attacks. 

Finally, we consider the Adult dataset. Figure~\ref{fig:4} illustrates the test performance of the trained model when using the logistic regression training in~\eqref{eqn:optimization:4_logistic_regression}  to mitigate the effect of label flipping attacks versus the percentage of the poisoned data~$\beta$ for various choices of regularization. The regularization here also offers robustness against poisoning attacks. 

\section{Conclusions and Future Work}
\label{sec:conc}

We use distributionally-robust optimization for machine learning to mitigate the effect of data poisoning attacks. We show that this problem can be cast as a regularized machine learning problem. We provide performance guarantees for the trained model on the correct data by training the model for the worst-case distribution on a neighbourhood around the empirical distribution (extracted from the training dataset corrupted by a poisoning attack) defined using the Wasserstein distance. Future work can focus on validation of the results on more general machine learning models, such as neural networks.

\bibliographystyle{plain}
\bibliography{citation}
\end{document}